\documentclass{amsart}
\usepackage[T1]{fontenc}
\usepackage[utf8]{inputenc}
\usepackage{amsmath,amssymb,amsthm}
\usepackage{xcolor}
\usepackage{cleveref}

\newtheorem{theorem}{Theorem}[section]
\newtheorem{lemma}[theorem]{Lemma}
\newtheorem{proposition}[theorem]{Proposition}
\newtheorem{corollary}[theorem]{Corollary}
\theoremstyle{definition}
\newtheorem{definition}[theorem]{Definition}
\newtheorem{example}[theorem]{Example}
\theoremstyle{remark}
\newtheorem{remark}[theorem]{Remark}

\newcommand{\pa}{\boldsymbol{\theta}}
\newcommand{\pasp}{\Theta}
\newcommand{\alt}[1]{\tilde{#1}}

\usepackage{soul}
\usepackage{graphicx}

\title[Polynomial Residual Networks are Generically Identifiable]{Polynomial Residual Networks \\ are Generically Identifiable}

\author{Isabel Dahlgren}
\address{Department of Mathematics, KTH, 114 28 Stockholm, Sweden}
\email{isabelda@kth.se}

\author{Giovanni Luca Marchetti}
\address{AI4S AB, Stockholm, Sweden}
\email{giovanni.marchetti@aiforscience.se}

\author{Kathlén Kohn}
\address{Department of Mathematics, KTH, 114 28 Stockholm, Sweden; Digital Futures, Osquars backe 5, 114 28 Stockholm, Sweden}
\email{kathlen@kth.se}

\date{\today}

\begin{document}

\maketitle

 \textbf{Abstract.} We prove identifiability results for constant-width residual networks with polynomial activations and bias terms. For activations of degree at least five with generic coefficients, the fiber over a generic parameter is a singleton. For activations of degree at least three with no linear term after a suitable translation of the input, the generic fiber is given by translating biases and then scaling weights and translated biases by real roots of unity. These results specialize to networks without bias terms. With the same proof technique, we further derive necessary conditions for critical points of the ResNet parametrization, that are central in singular learning theory.  

\section{Introduction}

Skip connections are standard components of contemporary neural network architectures. A skip connection allows information to bypass one or more layers of a network; in residual networks (ResNets), it adds the input of a block to its transformed output. Used early on in highway networks \cite{srivastava2015highway} and popularized by ResNets \cite{he2015deepresiduallearningimage}, skip connections now appear in virtually every large-scale model, from U-Nets \cite{ronneberger2015u} to Transformers \cite{vaswani2017attention}. This simple modification makes it possible to train networks with hundreds or even thousands of layers stably.

Mathematically, however, skip connections are poorly understood. Little is known about how they affect the parameter symmetries of a network, the geometry of its function space, or the critical points of its parametrization. These properties are well studied for many skip-connection free networks, but largely unexplored for residual ones.

A notable hypothesis is that skip connections ease training by breaking the symmetries and degeneracies of the parametrization
\cite{orhan2017skip}. So far, this hypothesis is supported mainly by experiments and simplified models. We study it rigorously for polynomial activations, asking which parameter symmetries survive the addition of skip connections and how the fibers and critical points of the parametrization change.

Polynomial activations are well suited to this question. Since polynomials approximate any continuous function arbitrarily well on compact sets, networks with generic high-degree polynomial activations are expected to reflect the behavior of general networks \cite{marchetti2025position}. At the same time, the network function is then polynomial in both the input variables and the parameters, so its fibers and differential can be studied by algebraic methods \cite{kileel2019expressivepowerdeeppolynomial,marchetti2025position}. This is the viewpoint of \textit{neuroalgebraic geometry}, which uses polynomial networks as a common framework for comparing architectures, such as convolutional \cite{shahverdi2024geometry,shahverdi2026learningrazorsedgeidentifiability,hendi2026geometry} and attention-based \cite{henry2024geometry} networks, and for isolating the effect of individual design choices, like the inclusion of skip connections.

We show that ResNets of arbitrary depth with a generic polynomial activation of degree at least five have no global parameter symmetries: a generic parameter is uniquely determined by its network function (Section~\ref{sec:identifiability}). In fully-connected networks without skip connections, by contrast, the neurons of each layer can be permuted without changing the network function. The global symmetries of such plain networks have been characterized for sigmoidal \cite{fefferman1994reconstructing,vlavcic2022neural} and hyperbolic tangent \cite{sussmann1992uniqueness,vlavcic2021affine} activations. For nonhomogeneous polynomial activations, proofs are recent and require large degrees \cite{kohn2026linear,hendi2026linear}; even for monomial activations, all known proofs have similar degree bounds \cite{kileel2019expressivepowerdeeppolynomial,finkel2025activation,cruaciun2025linear,usevich2025identifiability,massarenti2025alexander}. For comparison with this literature, we also determine all global parameter symmetries of ResNets with monomial activations.

We further study the critical points of the ResNet parametrization (Section~\ref{sec:criticality}). These are central objects in singular learning theory \cite{watanabe2009algebraic}. They can give rise to equilibria of the gradient flow, and thereby affect training dynamics \cite{amari2006singularities,wei2008dynamics,trager2019pure} 
as well as generalization 
\cite{wei2022deep,marchetti2025position}. They have been studied for plain networks with polynomial and monomial activations \cite{shahverdi2026learningrazorsedgeidentifiability,kohn2026singularlearningoccamsrazor}. We show that critical points do not arise away from degenerate parameters, and that they can cause singularities on the function space. 

All our results hold for ResNets with and without biases.

\section{Preliminaries}
\label{sec:preliminaries}

Throughout, the base field is $\mathbb{R}$, passing to $\mathbb{C}$ in proofs where needed, and $L , d$ are positive integers denoting the depth and width of networks, respectively. We write $I$ for the $d \times d$ identity matrix. Unless a more restrictive hypothesis is stated, the activation is a real polynomial \begin{equation}
\sigma(z)\mathrel{:=}\sum_{k = 0}^r a_k z^k , \quad a_k \in \mathbb{R} , \quad r\geq 2 , \quad a_r \neq 0 .
\label{eq:general-activation}
\end{equation} 
In the \textit{biased} setting, the biases are free parameters, and the parameter space is \[
\pasp\mathrel{:=}\left(\mathbb{R}^{d \times d} \times \mathbb{R}^d\right)^L , \quad \text{ with elements } \pa\mathrel{:=}\left(W_l , b_l\right)_{l = 1}^L .
\] In the \textit{unbiased} setting, every bias is fixed at zero, and we use $\pasp\mathrel{:=}(\mathbb{R}^{d \times d})^L$, with $\pa\mathrel{:=}(W_l)_{l = 1}^L$. 
Polynomial activations and powers of vectors are applied componentwise. Similarly, for two vectors $z , w \in \mathbb{R}^d$, $z * w$ denotes their Hadamard (componentwise) product. A polynomial vector has degree at most $N$ if every coordinate does.

\begin{definition}
A \textit{plain network} of depth $L$ and width $d$ is the function \mbox{$h_L:   \mathbb{R}^d\to \mathbb{R}^d$} defined recursively by \[
h_0 \left(x\right)\mathrel{:=} x , \quad h_l \left(x\right)\mathrel{:=}\sigma(W_l h_{l - 1} \left(x\right)+ b_l)
\]  for $1\leq l\leq L$. A \textit{ResNet} of depth $L$ and width $d$ is the function \mbox{$f_L: \mathbb{R}^d\to \mathbb{R}^d$} defined recursively by 
\begin{equation}
    f_0 \left(x\right)\mathrel{:=} x , \quad f_l \left(x\right)\mathrel{:=}\sigma(W_l f_{l - 1} \left(x\right)+ b_l)+ f_{l - 1} \left(x\right). \label{eq:residual}
\end{equation}
  The summand $f_{l - 1}$ is the \textit{skip connection}, and $W_l f_{l - 1} + b_l$ is the \textit{preactivation} of layer $l$.\end{definition}

Given a parameter $\pa \in \pasp$, we write $f_{\pa}$ for the ResNet $f_L$ defined as above. The \textit{ResNet parametrization} is the map \[
\psi : \pasp \rightarrow \mathbb{R}\left[x_1 ,\ldots , x_d\right]_{\leq r^L}^d , \quad \pa \mapsto f_{\pa}.
\] Here, $\mathbb{R}\left[x_1 ,\ldots , x_d\right]_{\leq r^L}$ is the finite-dimensional vector space of polynomials of degree at most $r^L$. 

\begin{definition}
The \textit{neuromanifold} is the image \[
\mathcal{M}\mathrel{:=}\psi(\pasp). 
\] 
 We call $\mathcal{M}$ the \textit{biased neuromanifold} when $\pasp =(\mathbb{R}^{d \times d} \times \mathbb{R}^d)^L$, and the \textit{unbiased neuromanifold} when $\pasp =(\mathbb{R}^{d \times d})^L$ and every $b_l = 0$. We use $\psi$ and $\mathcal{M}$ for the parametrization and neuromanifold in the stated setting. \end{definition}

Note that the parametrization  $\psi$ is not only polynomial in the input variables ($x_1 , \ldots , x_d$), but also in the parameters.
Therefore, the neuromanifold $\mathcal{M}$ is a semi-algebraic set by the Tarski--Seidenberg theorem. 
We will sometimes consider the Zariski closure $\overline{\mathcal{M}}$ of the neuromanifold, where the closure is taken in $\mathbb{R}\left[x_1 ,\ldots , x_d\right]_{\leq r^L}^d$, referring to $\overline{\mathcal{M}}$ as the \textit{neurovariety}.
The variety $\overline{\mathcal{M}}$ is irreducible: $\pasp$ is irreducible, and the closure of its image under a morphism is irreducible.

This work focuses on the fibers of the ResNet parametrization $\psi$, where the \textit{fiber over a parameter} $\pa \in \pasp$ is the preimage $\psi^{- 1}\left(\psi(\pa)\right)$. Fibers are taken in the real parameter space. When we pass to $\mathbb{C}$ in proofs, we say so explicitly and speak of the \textit{complex fiber}, the preimage of $\psi(\pa)$ under the complexified parametrization.
Our main results describe the generic fibers. A property of parameters is \textit{generic} if it holds on a nonempty Zariski-open subset of $\pasp$.

We further investigate critical points of the polynomial map $\psi$.
We denote by $J_{\psi} \left(\pa\right)$ the differential of $\psi$ at $\pa \in \pasp$.
A parameter $\pa$ is \textit{critical} if $\text{rank} J_{\psi} \left(\pa\right)$ is smaller than the maximum rank attained on $\pasp$; otherwise it is \textit{regular}. The maximum rank occurs on a nonempty Zariski-open set, because the rank conditions are given by minors.

\section{Identifiability}
\label{sec:identifiability}

The identifiability problem asks which parameters represent the same network function; that is, it asks for a description of the fibers of $\psi$. Our three main results describe the generic fiber, i.e., $\psi^{-1}(\psi(\pa))$ for generic $\pa \in \pasp$, for networks with and without biases. Our first theorem shows that,  for polynomial activations with generic coefficients, the fiber is a singleton in both settings. For monomial activations, and more generally the activations in the latter two theorems, the fiber consists of root-of-unity rescalings, with a bias translation when needed. 

\begin{theorem}
Let $\sigma$ have degree $r\geq 5$ and generic coefficients. Then the (complex) fiber over a generic parameter is a singleton. This holds for both the biased and the unbiased parametrization. \label[theorem]{thm:generic-activation} \end{theorem}

To state the remaining two results, we need some notation. For a polynomial $p\left(z\right)=\sum_{k = 0}^r c_k z^k$ of degree at least two, define \[
\delta(p)\mathrel{:=}\gcd\left\{k\geq 2 : c_k \neq 0\right\}, \; \mu_d \left(q\right)\mathrel{:=}\left\{\text{diag}\left(\omega_1 ,\ldots ,\omega_d\right):\omega_i^q = 1\right\} \subseteq \mathbb{C}^{d \times d}.
\] We write $\mu_d \left(q\right)_{\mathbb{R}}\mathrel{:=}\mu_d \left(q\right)\cap \mathbb{R}^{d \times d}$ for the real elements: diagonal sign matrices if $q$ is even, and $\left\{I\right\}$ if $q$ is odd.

Now, center the activation by setting \[
\tau\mathrel{:=}\frac{a_{r - 1}}{r a_r} , \quad \sigma_{\tau} \left(z\right)\mathrel{:=}\sigma(z -\tau)=\sum_{k = 0}^r c_k z^k .
\] By construction of $\tau$, we have $c_{r - 1} = 0$. Moreover, we write $\mathbf{1}\mathrel{:=}\left(1 ,\ldots , 1\right)^\top$.

\begin{theorem}
Let $\sigma$ have degree $r\geq 3$, and suppose $\sigma'(-\tau)= 0$. Then the fiber over a generic biased parameter $\pa$ is \[
\left\{\left(D_l W_l , D_l \left(b_l +\tau \mathbf{1}\right)-\tau \mathbf{1}\right)_{l = 1}^L : D_l \in \mu_d \left(\delta(\sigma_{\tau})\right)_{\mathbb{R}}\right\}.
\] \label[theorem]{thm:centred-activation} \end{theorem}
The condition $\sigma'(-\tau)= 0$ says that the centered activation $\sigma_{\tau}$ has zero linear coefficient $c_1$. The fiber is thus obtained by translating the biases by $\tau \mathbf{1}$, rescaling each weight row and the corresponding translated bias entry by the same real root of unity, and then undoing the translation. 
For unbiased networks, the degree bound can be lowered to two.
\begin{theorem}
Let $\sigma$ have degree $r\geq 2$, and suppose $\sigma(0)=\sigma'(0)=0$. Then the fiber over a generic unbiased parameter $\pa$ is \[
\left\{\left(D_l W_l\right)_{l = 1}^L : D_l \in \mu_d \left(\delta(\sigma)\right)_{\mathbb{R}}\right\}.
\] \label[theorem]{thm:identifiability-unbiased} \end{theorem}
The last two theorems apply in particular to monomial activations:
\begin{corollary}\label[corollary]{cor:monomial-activation}
    For $\sigma(z)= z^r$, $r \geq 3$, the fiber over a generic biased parameter $\pa$ is \[
\left\{\left(D_l W_l , D_l b_l \right)_{l = 1}^L : D_l \in \mu_d \left(r\right)_{\mathbb{R}}\right\}.
\] For $\sigma(z)=z^r$, $r\geq 2$, the fiber over a generic unbiased parameter $\pa$ is \[
\left\{\left(D_1 W_1,\ldots,D_L W_L\right):D_l\in\mu_d(r)_{\mathbb{R}}\right\}.
\]
\end{corollary}
\begin{proof}
    In the biased case, we have $\tau = 0$ and $\delta(\sigma_{\tau})= r$, so the claim follows from \Cref{thm:centred-activation}. In the unbiased case, $\delta(\sigma)=r$ and $\sigma(0)=\sigma'(0)=0$, so \Cref{thm:identifiability-unbiased} applies.
\end{proof}
Our fiber theorems determine the dimension of the neuromanifold, which measures the expressivity and sample complexity of the network \cite{marchetti2025position}:
\begin{corollary}
Under the hypotheses of \Cref{thm:generic-activation,thm:centred-activation} on the activation, the biased neuromanifold has dimension $L\left(d^2 + d\right)$. Under the hypotheses of \Cref{thm:generic-activation,thm:identifiability-unbiased} on the activation, the unbiased neuromanifold has dimension $L d^2$.  \label[corollary]{cor:dimension} \end{corollary}
\begin{proof}
In each case, the parametrization $\psi : \pasp \rightarrow \overline{\mathcal{M}}$ is dominant between irreducible varieties and has generically finite fibers by the corresponding identifiability result. The fiber-dimension theorem \cite[Sec. 12.4]{vakil2025rising} gives $\dim \pasp=\dim \overline{\mathcal{M}} = \dim \mathcal{M}$.
\end{proof}

\Cref{thm:generic-activation,thm:centred-activation} follow from our master theorem on fibers, \Cref{thm:biased-fiber}, which describes the fiber over every parameter in an explicit Zariski-open set and counts its elements. We prove them in \Cref{sec:proof-main}, and prove \Cref{thm:identifiability-unbiased} in \Cref{sec:proof-unbiased}.

\subsection{Comparison with plain networks}
For plain networks with monomial activation of sufficiently large degree $r$, the generic fiber of the network parametrization consists of parameters obtained by neuron permutations and neuronwise rescalings in each hidden layer \cite{finkel2025activation, usevich2025identifiability}.
The known degree bounds are quadratic in the network width; for example, \cite{finkel2025activation} gives the bound $r > 6D^2 - 6D + 1$, where $D := 2d - 1$.
In contrast, \Cref{cor:monomial-activation} characterizes the generic fiber for ResNets for $r \geq 3$ in the biased case and $r \geq 2$ in the unbiased case, independently of depth and width.
The finite fibers in this result show that skip connections break the continuous parameter symmetries present in plain networks with monomial activations.

For unbiased plain networks with generic polynomial activation, it was recently shown that the generic fiber of the network parametrization map consists of parameters obtained by neuron permutations in each hidden layer \cite{hendi2026linear}.
Their proof requires a degree bound similar to the monomial case, namely $r \geq 10 d^2 -9d+4$.
For deep ResNets, \Cref{thm:generic-activation} provides the generic fiber for $r \geq 5$, a bound independent of depth and width.
This result shows that skip connections reduce parameter symmetries, yielding generic uniqueness.
We note that our proof techniques for ResNets are simpler than those of \cite{hendi2026linear} for plain networks.

\subsection{Proof of the main theorems}
\label{sec:proof-main}

We now prove \Cref{thm:generic-activation} and \Cref{thm:centred-activation}. To state the master theorem on fibers from which both follow, we first introduce two conditions on parameters: nondegeneracy and layer independence.

Given a ResNet with an activation function $\sigma$ of degree $r \geq 2$, the auxiliary \textit{monomial plain network} has outputs $g_l$ and preactivations $q_l$ defined by \begin{equation}
g_0 \left(x\right)\mathrel{:=} x , \quad q_l \left(x\right)\mathrel{:=} W_l g_{l - 1} \left(x\right), \quad g_l \left(x\right)\mathrel{:=} q_l \left(x\right)^r
\label{eq:monomial-network}
\end{equation} for $1\leq l\leq L$. Its coordinates $\left(g_l\right)_i$ are called \textit{neurons}; a neuron is \textit{dead} if it is the zero polynomial. Thus, this network is obtained from \eqref{eq:residual} by deleting the skip connections and biases and replacing $\sigma(z)$ by $z^r$.

\begin{definition}
A parameter $\pa$ is \textit{nondegenerate} if $\left(g_l\right)_i \neq 0$ for every layer $1\leq l\leq L$ and coordinate $1\leq i\leq d$ in \eqref{eq:monomial-network}. Otherwise it is \textit{degenerate}. \label[definition]{def:nondegenerate} \end{definition}

The definition depends on the activation degree $r$, which is fixed throughout each statement.

We point out that nondegeneracy is a generic property. Indeed, the equation $\left(g_l\right)_i = 0$ sets finitely many polynomial functions in the network weights equal to zero, one for each coefficient in $\left(g_l\right)_i$. The degenerate locus is the finite union of these closed sets. It is proper because the identity weights (with zero biases) are nondegenerate, as is easily verified.

The following statements use the notation introduced in \eqref{eq:general-activation} and \eqref{eq:residual}.

\begin{lemma}
Let $\pa$ be nondegenerate. At each layer $l$, every coordinate of $W_l f_{l - 1} + b_l$ has degree $r^{l - 1}$, and every coordinate of $f_l$ is a polynomial of degree $r^l$. Their homogeneous components of these degrees are $\kappa_{l - 1} q_l$ and $\kappa_l g_l$, respectively, where \[
\kappa_0\mathrel{:=} 1 , \quad \kappa_l\mathrel{:=} a_r \kappa_{l - 1}^r .
\] Moreover, every row of every $W_l$ is nonzero. \label[lemma]{lem:leading-components} \end{lemma}

\begin{proof}
The formulas for homogeneous components clearly hold for $l = 1$. For $l\geq 2$, assume that the leading components of $W_{l - 1} f_{l - 2} + b_{l - 1}$ and $f_{l - 1}$ are $\kappa_{l - 2} q_{l - 1}$ and $\kappa_{l - 1} g_{l - 1}$, respectively. Now recall that $q_l = W_l g_{l - 1}$ and $g_l = q_l^r$. Therefore the leading component of $W_l f_{l - 1} + b_l$ is $W_l \kappa_{l - 1} g_{l - 1} = \kappa_{l - 1} q_l$, while the leading component of $f_l$ is $a_r \left(\kappa_{l - 1} q_l\right)^r = \kappa_l g_l$. Nondegeneracy and $\kappa_l \neq 0$ give both degree assertions.

A zero row of $W_l$ would make a coordinate of $q_l$, hence of $g_l$, zero. \end{proof}

\begin{definition}
A parameter is \textit{layer-independent} if the coordinates of $f_l$ are linearly independent for every $0\leq l < L$. It is \textit{strongly layer-independent} if these coordinates together with the constant polynomial $1$ are linearly independent for every $0\leq l < L$. \label[definition]{def:layer-independence} \end{definition}

Both properties are generic, as we now justify. Writing each $\left(f_l\right)_i$ in the monomial basis, (strong) layer independence means that the resulting coefficient matrix has full column rank, which is a Zariski-open condition. Moreover, the identity weights with zero biases are (strongly) layer-independent. Indeed, there, each coordinate $\left(f_l\right)_i$ depends only on $x_i$ and has degree $r^l$ with nonzero leading coefficient. Comparing the coefficients of $x_i^{r^l}$ in an affine relation proves strong layer independence. All in all, (strong) layer independence and nondegeneracy hold simultaneously on a nonempty Zariski-open set of network weights.

We say  a polynomial $p\left(z\right)=\sum_{k = 0}^r c_k z^k$ is \textit{invariant under} $\mu_d \left(q\right)$ if  $p\left(D z\right)= p\left(z\right)$ for all $D \in \mu_d \left(q\right)$, applying $p$ componentwise. Invariance under $\mu_d \left(\delta(p)\right)$ holds precisely when $c_1 = 0$ or $\delta(p)= 1$: all nonlinear terms are already invariant, while a nonzero linear term requires every allowed root to equal $1$.

We are now ready to state our master theorem on fibers.

\begin{theorem}
Let $\sigma$ have degree $r\geq 3$, and suppose $\sigma_{\tau}$ is invariant under $\mu_d \left(\delta(\sigma_{\tau})\right)$. For a nondegenerate, strongly layer-independent biased parameter $\pa$, its fiber is exactly \[
\left\{\left(D_l W_l , D_l \left(b_l +\tau \mathbf{1}\right)-\tau \mathbf{1}\right)_{l = 1}^L : D_l \in \mu_d \left(\delta(\sigma_{\tau})\right)_{\mathbb{R}}\right\}.
\] It has $2^{L d}$ elements if $\delta(\sigma_{\tau})$ is even, and is a singleton if $\delta(\sigma_{\tau})$ is odd. The complex fiber is given by the same formula with $D_l \in \mu_d \left(\delta(\sigma_{\tau})\right)$ and has $\delta(\sigma_{\tau})^{L d}$ elements. \label[theorem]{thm:biased-fiber} \end{theorem}

\noindent
Before we show this theorem, we explain how it proves \Cref{thm:generic-activation,thm:centred-activation}.

\begin{proof}[Proof of \Cref{thm:generic-activation}]
The coefficients $c_2 , c_3$ of $\sigma_\tau$ are rational functions of $a_0 ,\ldots ,$ $a_r$, with denominators powers of the nonzero coefficient $a_r$. Their simultaneous nonvanishing defines a nonempty Zariski-open set. On this set, $\delta(\sigma_{\tau})= 1$, so the invariance hypothesis is automatic and \Cref{thm:biased-fiber} gives a singleton generic fiber in the biased setting. 
The fibers under the unbiased parametrization are subsets of the biased fibers. 
Since a generic unbiased parameter is nondegenerate and strongly layer-independent, its fiber is a singleton as well.
\end{proof}

\begin{proof}[Proof of \Cref{thm:centred-activation}]
The linear coefficient of $\sigma_{\tau}$ is $c_1 =\sigma'(-\tau)= 0$, so $\sigma_{\tau}$ is invariant under $\mu_d \left(\delta(\sigma_{\tau})\right)$. 
Therefore, \Cref{thm:biased-fiber} gives the generic fiber. \end{proof}

The proof of \Cref{thm:biased-fiber} proceeds by induction, recovering the final preactivation by comparing degrees and then removing the final layer. The scalar comparison below is the only step that uses the centering of the activation.

\begin{lemma}
Let $s\left(z\right)=\sum_{k = 0}^r c_k z^k$ have degree $r\geq 3$, with $c_{r - 1} = 0$, and be invariant under $\mu_d \left(\delta(s)\right)$. Suppose $p ,\alt{p}, u ,\alt{u} \in \mathbb{C}\left[x_1 ,\ldots , x_d\right]$ satisfy $\deg p = R\geq 1$, $\deg \alt{p},\deg u ,\deg \alt{u}\leq R$, and \[
s\left(p\right)+ u = s\left(\alt{p}\right)+\alt{u}.
\] Then $\alt{p}=\omega p$ for some $\omega^{\delta(s)} = 1$, and $\alt{u}= u$. \label[lemma]{lem:centred-comparison} \end{lemma}

\begin{proof}
Let $h$ and $\alt{h}$ denote the leading components of $p$ and $\alt{p}$, respectively. Comparing terms of degree $r R$ gives $\alt{h}^r = h^r$, which yields $\alt{h}=\omega h$ with $\omega^r = 1$.

Suppose $q\mathrel{:=}\omega p -\alt{p}$ is nonzero, and set $m\mathrel{:=}\deg q < R$. The factorization \[
p^r -\alt{p}^r = \left(\omega p\right)^r -\alt{p}^r = q \sum_{j = 0}^{r - 1}\left(\omega p\right)^{r - 1 - j}\alt{p}^j
\] has degree $m +\left(r - 1\right)R$, since the sum has leading component $r\left(\omega h\right)^{r - 1} \neq 0$. Since $c_{r - 1} = 0$, all remaining terms of $s\left(p\right)- s\left(\alt{p}\right)$ have degree at most $\left(r - 2\right)R$. Consequently the difference has degree $m +\left(r - 1\right)R > R$, contradicting its equality with $\alt{u}- u$. Thus $\alt{p}=\omega p$.

If $c_k (1 -\omega^k) \neq 0$ for some $k\geq 2$, the largest such $k$ makes $s\left(p\right)- s\left(\omega p\right)$ have degree $k R > R$, another contradiction. Hence $\omega^{\delta(s)} = 1$. Invariance gives $s\left(\omega p\right)= s\left(p\right)$, and the assumed equality now gives $\alt{u}= u$. \end{proof}

\begin{proof}[Proof of \Cref{thm:biased-fiber}]
 Replace each bias by $\beta_l\mathrel{:=} b_l +\tau \mathbf{1}$ and the activation by $\sigma_{\tau}$. We work over $\mathbb{C}$ and obtain the real fiber by intersecting with the real parameter space. This replacement leaves every layer output unchanged, since $\sigma_{\tau} \left(W_l f_{l - 1} +\beta_l\right)=\sigma(W_l f_{l - 1} + b_l)$. We recover these shifted parameters by induction on depth.

Suppose an arbitrary shifted parameter $\left(\alt{W}_l ,\alt{\beta}_l\right)_{l = 1}^L$ has the same final output. Apply \Cref{lem:centred-comparison} coordinatewise with $R\mathrel{:=} r^{L - 1}$ to the final preactivations and preceding outputs. The recursion bounds the degrees of both preactivations and preceding outputs $f_{L - 1}$ and $\alt{f}_{L - 1}$ by $R$. We obtain \[
\alt{f}_{L - 1} = f_{L - 1} , \quad \alt{W}_L f_{L - 1} +\alt{\beta}_L = D_L \left(W_L f_{L - 1} +\beta_L\right)
\] for some $D_L \in \mu_d \left(\delta(\sigma_{\tau})\right)$. Strong layer independence gives $\alt{W}_L = D_L W_L$ and $\alt{\beta}_L = D_L \beta_L$. For $L = 1$ this immediately recovers the shifted parameters. For $L > 1$ apply the induction hypothesis to the truncated parameter $\left(\alt{W}_l , \alt{\beta}_l\right)_{l = 1}^{L - 1}$, which remains nondegenerate and strongly layer-independent. Subtracting $\tau \mathbf{1}$ from each recovered bias gives the displayed fiber.

Conversely, the invariance of $\sigma_{\tau}$ ensures that each transformation in the theorem statement preserves every layer output. Thus every resulting parameter lies in the fiber.

Every weight row is nonzero by \Cref{lem:leading-components}; hence distinct choices of the $L d$ roots give distinct parameters, which yields the count of the complex fiber. A transformed parameter is real exactly when each root is real, because it rescales a nonzero real weight row. The only real roots of unity are $\pm 1$, and $-1$ has order dividing $\delta(\sigma_{\tau})$ if and only if $\delta(\sigma_{\tau})$ is even. \end{proof}

\subsection{Proof of the unbiased theorem}
\label{sec:proof-unbiased}

Throughout this subsection, we use the unbiased parametrization and the hypotheses of \Cref{thm:identifiability-unbiased}: $r\geq 2$ and $\sigma(0)=\sigma'(0)=0$. In other words, the constant and linear terms of the activation vanish, so we may write \begin{equation}
\sigma(z)\mathrel{:=}\sum_{k = 2}^r a_k z^k , \quad r\geq 2 , \quad a_r \neq 0 .
\label{eq:activation}
\end{equation}
Write $\delta\mathrel{:=}\delta(\sigma)$. Since every exponent occurring in $\sigma$ is divisible by $\delta$, the activation is invariant under $\mu_d\left(\delta\right)$. As before, we work over $\mathbb{C}$ and recover the real fiber by intersecting with the real parameter space.

The vanishing constant and linear terms ensure layer independence at every parameter. We first prove this, then use a degree comparison to recover the weights one layer at a time.

\begin{lemma}
For unbiased networks with activation \eqref{eq:activation}, the coordinates of $f_l$ are algebraically independent for every parameter and every $0\leq l\leq L$. In particular, these networks are strongly layer-independent at every parameter. \label[lemma]{lemma:lin-indep-fl} \end{lemma}

\begin{proof}
The recursion defining $f_l$ preserves the linear part and the value at the origin: we may write $f_l \left(x\right)= s_l \left(x\right) + x$, where all monomials of $s_l$ have degree at least two. Suppose for contradiction that a nonzero polynomial $P$ satisfies $P\left(f_l\right)= 0$, and let $Q$ be its lowest-degree homogeneous component. Then the lowest-degree homogeneous component of $P\left(s_l \left(x\right) + x\right)$ is $Q \left(x\right) \neq 0$, a contradiction. \end{proof}

We apply the following scalar degree comparison to every coordinate of the final preactivation.

\begin{lemma}
Let $p ,\alt{p}, u ,\alt{u} \in \mathbb{C}\left[x_1 ,\ldots , x_d\right]$, where $\deg p = R\geq 1$, $\deg \alt{p}$, $\deg u$, $\deg \alt{u}\leq R$, and $p\left(0\right)=\alt{p}\left(0\right)= 0$. Under the activation hypotheses of this subsection, if \begin{equation}
\sigma(p)+ u =\sigma(\alt{p})+\alt{u},
\label{eq:comparison}
\end{equation} then $\alt{p}=\omega p$ for some $\omega^{\delta} = 1$, and $\alt{u}= u$. \label[lemma]{lem:comparison} \end{lemma}

\begin{proof}
The proof of \Cref{lem:centred-comparison} applies with only the following change in the degree estimate. If $q\mathrel{:=}\omega p -\alt{p}$ is nonzero, then $q\left(0\right)= 0$, so $m\mathrel{:=}\deg q\geq 1$. The same factorization gives \[
\deg(p^r -\alt{p}^r)= m +\left(r - 1\right)R >\left(r - 1\right)R\geq R .
\] All remaining terms of $\sigma(p)-\sigma(\alt{p})$ have degree at most $\left(r - 1\right)R$. This again contradicts \eqref{eq:comparison}, so $\alt{p}=\omega p$. The argument giving $\omega^{\delta} = 1$ and $\alt{u}= u$ is unchanged. \end{proof}

\begin{proof}[Proof of \Cref{thm:identifiability-unbiased}]
We prove the stated fiber description for every nondegenerate unbiased parameter $\pa$. By \Cref{lemma:lin-indep-fl}, $\pa$ is layer-independent. Suppose an arbitrary unbiased parameter $\alt{\pa}$ has the same final output. Since $\sigma(0)=0$, all layer outputs and preactivations of both networks vanish at the origin. By \Cref{lem:leading-components}, each coordinate of $W_L f_{L-1}$ has degree $R\mathrel{:=}r^{L-1}$; the recursion bounds the degrees of both final preactivations and preceding outputs by $R$. Applying \Cref{lem:comparison} coordinatewise gives \[
\alt{f}_{L - 1} = f_{L - 1} , \quad \left(\alt{W}_L - D_L W_L\right)f_{L - 1} = 0
\] for some $D_L \in \mu_d \left(\delta\right)$. Layer independence gives $\alt{W}_L = D_L W_L$. Repeating this argument for the preceding layers recovers $\alt{W}_l=D_l W_l$ for every $l$, ending with $f_0=x$.

Conversely, invariance of $\sigma$ under $\mu_d\left(\delta\right)$ shows that these row rescalings preserve every layer output. Every weight row of $\pa$ is nonzero by \Cref{lem:leading-components}, so a rescaled parameter is real exactly when each rescaling factor is real. This gives the stated real fiber at every nondegenerate parameter. Nondegeneracy holds generically, proving the theorem. \end{proof}

\begin{remark}
    For activations as in \eqref{eq:activation} -- in particular, for nonlinear monomials -- the fiber description in \Cref{thm:identifiability-unbiased} holds for every nondegenerate parameter. Degenerate parameters are similar to subnetworks, in that they yield networks which may be pruned. The above result can thus be understood as the analog of \cite[Thm. A.2]{kohn2026singularlearningoccamsrazor} about monomial plain networks, which asserts that the generic fiber is attained away from subnetworks.
    
    For monomial plain networks it is readily checked that subnetworks have larger fibers. This, however, is not true for degenerate parameters of monomial ResNets, as the following subsection shows.
\end{remark}

\subsection{Fibers over degenerate parameters} \label{sec:degfib}

The preceding identifiability results apply to nondegenerate parameters; we now study fibers over degenerate parameters, focusing on unbiased networks with $L = d = 2$ and $\sigma(z)\mathrel{:=} z^2$ over $\mathbb{R}$.

The \textit{row-sign orbit} of $\pa = \left(A , B\right)$ consists of independent sign changes of its rows. By \Cref{thm:identifiability-unbiased} and \Cref{lemma:lin-indep-fl}, the real fiber of a nondegenerate parameter is exactly this orbit. At degenerate parameters, fibers can contain further row-sign orbits or even continuous families.

Identifiability at degenerate parameters requires a more detailed analysis. Degeneracy does not always enlarge the fiber: the real fiber of $(0 , 0)$ is a singleton, since an output equal to $x$ forces the sum of the componentwise squares $(B(A x)^2 + B x)^2 +(A x)^2$ to vanish, hence $A = B = 0$. If exactly one matrix is zero, exchanging the two matrices preserves the output and gives a distinct row-sign orbit. The following three examples illustrate further ways in which fibers can grow. All displayed scalar parameters are nonzero and real.

\begin{example}
Take \[
A\mathrel{:=}\begin{pmatrix}
 0 & 0 \\
 \alpha & 0
\end{pmatrix} , \quad B\mathrel{:=}\begin{pmatrix}
 \gamma & 0 \\
 \beta & 0
\end{pmatrix} .
\] The first coordinate is unchanged by the first layer, so \[
f_2 \left(x\right)=\begin{pmatrix}
 \gamma^2 x_1^2 + x_1 \\
 \left(\alpha^2 +\beta^2\right)x_1^2 + x_2
\end{pmatrix} .
\] With $\gamma$ fixed, the output determines only $\alpha^2 +\beta^2$. Thus every pair satisfying \begin{equation}
\alt{\alpha}^2 +\alt{\beta}^2 =\alpha^2 +\beta^2
\label{eq:fermat-fiber}
\end{equation} gives the same network. This circle contains parameters outside the finite row-sign orbit. \label[example]{ex:fermat-fiber} \end{example}

\begin{example}
Take \[
A\mathrel{:=}\begin{pmatrix}
 0 & 0 \\
 \alpha & \beta
\end{pmatrix} , \quad B\mathrel{:=}\begin{pmatrix}
 \gamma & 0 \\
 0 & 0
\end{pmatrix} , \quad C\mathrel{:=}\begin{pmatrix}
 \gamma & 0 \\
 \alpha & \beta
\end{pmatrix} .
\] The first layer changes only the second coordinate, while the second layer uses only the first coordinate. Hence $\left(A , B\right)$ has output \[
f_2 \left(x\right)=\begin{pmatrix}
 \gamma^2 x_1^2 + x_1 \\
 \left(\alpha x_1 +\beta x_2\right)^2 + x_2
\end{pmatrix} .
\] This is also the output of $\left(0 , C\right)$ and $\left(C , 0\right)$: the two nonzero rows can be consolidated in either layer. These three parameters lie in distinct row-sign orbits. \label[example]{ex:layer-consolidation} \end{example}

\begin{example}
Take \[
A\mathrel{:=}\begin{pmatrix}
 \alpha & \beta \\
 0 & 0
\end{pmatrix} , \quad B\mathrel{:=}\begin{pmatrix}
 \delta & 0 \\
 0 & \gamma
\end{pmatrix} , \
 \alt{A}\mathrel{:=}\begin{pmatrix}
 \alpha & \beta \\
 0 & \gamma
\end{pmatrix} , \quad \alt{B}\mathrel{:=}\begin{pmatrix}
 \delta & 0 \\
 0 & 0
\end{pmatrix} .
\] Writing $h\left(x\right)\mathrel{:=}\left(\alpha x_1 +\beta x_2\right)^2 + x_1$, both $\left(A , B\right)$ and $\left(\alt{A},\alt{B}\right)$ have output \[
f_2 \left(x\right)=\begin{pmatrix}
 \delta^2 h\left(x\right)^2 + h\left(x\right) \\
 \gamma^2 x_2^2 + x_2
\end{pmatrix} .
\] Thus the second row of $B$ can move to $A$ without changing the network. The two parameters lie in distinct row-sign orbits, since their zero rows occur in different layers. \label[example]{ex:row-transfer} \end{example}

A classification of fibers on the degenerate locus remains open.

\section{Criticality}
\label{sec:criticality}

The degree comparisons in the arguments for identifiability also give sufficient conditions for regularity of the ResNet parametrization. Here, no invariance hypothesis (as in \Cref{thm:biased-fiber}) is imposed on the activation.

\begin{proposition}
Let $\sigma$ have degree $r\geq 3$. At every nondegenerate, strongly layer-independent biased parameter $\pa$, the differential $J_{\psi} \left(\pa\right)$ is injective. In other words, every critical parameter is degenerate or fails strong layer independence. \label[proposition]{prop:regularity-bias} \end{proposition}

\begin{proposition}
Let $\sigma$ have degree $r\geq 2$ and satisfy $\sigma(0)= 0$. At every nondegenerate, layer-independent unbiased parameter $\pa$, the differential $J_{\psi} \left(\pa\right)$ is injective. For activations of the form \eqref{eq:activation}, every critical unbiased parameter is degenerate. \label[proposition]{prop:regularity-unbiased} \end{proposition}

Choose a parameter tangent $(\dot{W}_l ,\dot{b}_l)_{l = 1}^L$ and vary each weight and bias along the corresponding affine line. A dot on a layer output denotes its derivative at the original parameter, with $x$ fixed. For unbiased networks all biases and their derivatives are zero. Injectivity means that $\dot{f}_L = 0$ forces the parameter tangent to vanish.

\begin{lemma}
At a nondegenerate parameter, $\dot{f}_L = 0$ implies $\dot{p} = 0$ for $p\mathrel{:=} W_L f_{L - 1} + b_L$ in either of the following settings: biased networks with $r\geq 3$, or unbiased networks with $r\geq 2$ and $\sigma(0)= 0$. \label[lemma]{lem:preactivation-derivative} \end{lemma}

\begin{proof}
We compare the degrees of the two terms in the chain rule \begin{equation}
0 =\dot{f}_L =\sigma'\left(p\right)*\dot{p} +\dot{f}_{L - 1} .
\label{eq:critical-chain}
\end{equation} Set $R\mathrel{:=} r^{L - 1}$. By \Cref{lem:leading-components}, $\deg \sigma'\left(p_i\right)=\left(r - 1\right)R$ for every $i$. In the biased setting, a nonzero $\dot{p}_i$ therefore makes the product have degree at least $\left(r - 1\right)R > R$. In the unbiased setting, $\sigma(0) = 0$ yields $p\left(0\right)= 0$ for each parameter, so a nonzero $\dot{p}_i$ has degree at least one and the product has degree at least $\left(r - 1\right)R + 1 > R$. Neither can cancel $\dot{f}_{L - 1}$, which has degree at most $R$. Thus $\dot{p} = 0$. \end{proof}

\begin{proof}[Proof of 
\Cref{prop:regularity-bias,prop:regularity-unbiased}] We recover the tangent one layer at a time, starting at the output. Suppose $\dot{f}_L = 0$. By \Cref{lem:preactivation-derivative}, $\dot{p} = 0$, and \eqref{eq:critical-chain} gives $\dot{f}_{L - 1} = 0$. The product rule becomes \[
0 =\dot{p} =\dot{W}_L f_{L - 1} +\dot{b}_L .
\] Strong layer independence forces $\dot{W}_L = 0$ and $\dot{b}_L = 0$. For the unbiased case, layer independence suffices to give $\dot{W}_L = 0$. Repeat with the preceding layer. This reaches $f_0 = x$, whose coordinates together with $1$ are linearly independent, and proves that every parameter derivative vanishes.

An injective differential has maximal rank, so the assertions about critical parameters follow. For activations as in \eqref{eq:activation}, layer independence holds everywhere by \Cref{lemma:lin-indep-fl}. \end{proof}

\begin{remark} \label[remark]{rem:subnetCrit}
    For plain unbiased monomial networks of activation degree $r  > 9(d(d-1))^{L-1}$, the network parametrization is regular away from subnetworks \cite{kohn2026singularlearningoccamsrazor}.
    Specializing \Cref{prop:regularity-unbiased} to monomial networks, we obtain a corresponding statement for unbiased monomial ResNets. Here, the result holds for $r \geq 2$.
\end{remark}

In the plain network setting from \Cref{rem:subnetCrit}, all subnetworks are critical points of the network parametrization, since they cause fibers of larger dimension than the generic fiber.
However, for unbiased monomial ResNets, not all degenerate parameters are critical points, as illustrated by shallow networks of width two below. 
Before discussing this situation in \Cref{sec:shallow-criticality}, we provide a simple sufficient criterion for criticality.

\begin{lemma}
For unbiased networks with activation \eqref{eq:activation}, a parameter with a zero weight row is critical. \label[lemma]{lem:zero-row-critical} \end{lemma}

\begin{proof}
Suppose row $i$ of $W_l$ is zero. Choose a nonzero row vector $v^\top$ and vary only that row: $\dot{W}_l = e_i v^\top$, and $\dot{W}_k = 0$ for $k \neq l$. Then $\dot{f}_{l - 1} = 0$, so \[
\dot{f}_l =\sigma'\left(W_l f_{l - 1}\right)*\dot{W}_l f_{l - 1} .
\] Only coordinate $i$ can be nonzero, and it vanishes because its preactivation is zero and $\sigma'\left(0\right)= 0$. All subsequent weight matrices are fixed, so the chain rule gives $\dot{f}_L = 0$. This nonzero parameter tangent lies in $\ker J_{\psi} \left(\pa\right)$. Since the Jacobian is generically injective (\Cref{prop:regularity-unbiased}), $\pa$ is a critical point. \end{proof}

\subsection{Two-layer networks of width two}
\label{sec:shallow-criticality}


Throughout this subsection the networks are unbiased. Let $L = d = 2$ and $\sigma(z)= z^r$ with $r\geq 3$. We classify the critical points of the ResNet parametrization and write $f$ in place of $f_2$. The arguments are carried out over $\mathbb{C}$, which is harmless: the Jacobian of $\psi$ at a real parameter is a real matrix, so its rank is the same over $\mathbb{R}$ and $\mathbb{C}$.

Write the weights as $\left(A , B\right)$, with row vectors $a_i$ and $b_i$. Set \[
w\mathrel{:=} A x , \quad u\mathrel{:=} B w^r , \quad v\mathrel{:=} B x .
\] Thus $u$ is the second preactivation of the monomial plain network, whereas $u + v$ is the second residual preactivation. The output and its derivative are \begin{align}
f\left(x\right)& =\left(u + v\right)^r + w^r + x , \\
 \dot{f} & = r\left(u + v\right)^{r - 1} *\left(\dot{u} +\dot{v}\right) + r w^{r - 1} *\dot{w} , \label{eq:shallow-derivative}
\end{align} where \[
\dot{w} =\dot{A} x , \quad \dot{v} =\dot{B} x , \quad \dot{u} =\dot{B} w^r + r B\left(w^{r - 1} *\dot{w}\right).
\]

\begin{proposition}
For $L = d = 2$ and $\sigma(z)= z^r$, $r\geq 3$, an unbiased parameter $\left(A , B\right)$ is critical if and only if a weight matrix has a zero row or \begin{equation}
A =\rho \begin{pmatrix}
 1 & -\lambda^{- r} \\
 \lambda & -\lambda^{1 - r}
\end{pmatrix} , \quad B =\begin{pmatrix}
 \mu & -\mu \lambda^{- r} \\
 \nu & -\nu \lambda^{- r}
\end{pmatrix}
\label{eq:shallow-critical-form}
\end{equation} for some nonzero $\rho ,\lambda ,\mu ,\nu \in \mathbb{R}$. \label[proposition]{prop:shallow-critical-classification} \end{proposition}

\begin{remark}
    Not all degenerate parameters are critical. For instance, take \[
    A := \begin{pmatrix}
     1 & 0 \\
     1 & 0
    \end{pmatrix} , \quad B := \begin{pmatrix}
     1 & -1 \\
     1 & 1
    \end{pmatrix} .
    \] Then $u = B(Ax)^r = (0, 2x_1^r)^\top$, so the parameter is degenerate. However, neither matrix has a zero row, and $B$ is invertible, whereas every $B$ in \eqref{eq:shallow-critical-form} is singular. Thus \Cref{prop:shallow-critical-classification} shows that this parameter is not critical for any $r \geq 3$.
\end{remark}

We prove necessity through three kernel calculations. The first separates coordinates according to whether the monomial preactivation $u_i$ vanishes.

\begin{lemma}
Suppose all rows of $A$ and $B$ are nonzero and $\left(\dot{A} ,\dot{B}\right) \in \ker J_{\psi} \left(A , B\right)$. For each $i$:

 \begin{enumerate}
 \item If $u_i \neq 0$, then $\dot{a}_i =\dot{b}_i = 0$.
 \item If $u_i = 0$, then $\dot{u}_i = 0$ and \begin{equation}
v_i^{r - 1}\dot{v}_i + w_i^{r - 1}\dot{w}_i = 0 .
\label{eq:shallow-dead-row}
\end{equation} If $v_i , w_i$ are not proportional, then $\dot{a}_i =\dot{b}_i = 0$.
 \end{enumerate}
\label[lemma]{lem:shallow-coordinate-test} \end{lemma}

\begin{proof}
Divide the $i$th coordinate of \eqref{eq:shallow-derivative} by $r$. If $u_i \neq 0$, its degree-$r^2$ component is $u_i^{r - 1}\dot{u}_i$, so $\dot{u}_i = 0$. Its component of degree $r\left(r - 1\right)+ 1$ is then $u_i^{r - 1}\dot{v}_i$, so $\dot{v}_i = 0$. The remaining equation is $w_i^{r - 1}\dot{w}_i = 0$, giving $\dot{w}_i = 0$.

If $u_i = 0$, the equation becomes \[
v_i^{r - 1}\dot{u}_i + v_i^{r - 1}\dot{v}_i + w_i^{r - 1}\dot{w}_i = 0 .
\] The first term is homogeneous of degree $2 r - 1$, and the other terms are homogeneous of degree $r$. Since $v_i \neq 0$, this gives $\dot{u}_i = 0$ and \eqref{eq:shallow-dead-row}. If $v_i , w_i$ are nonproportional, they are coprime linear forms, so $v_i^{r - 1}$ divides $\dot{w}_i$. For $r\geq 3$ this forces the linear form $\dot{w}_i$ to vanish, and then $\dot{v}_i = 0$ as well. \end{proof}

\begin{lemma}
Suppose all weight rows are nonzero and $w_2 =\lambda w_1$ for some $\lambda \neq 0$. If exactly one of $u_1 , u_2$ vanishes, then $J_{\psi} \left(A , B\right)$ is injective. \label[lemma]{lem:one-dead-neuron} \end{lemma}

\begin{proof}
We give the calculation for $u_1 = 0$, $u_2 \neq 0$; interchanging the two coordinate indices gives the other case. Take a kernel tangent $\left(\dot{A} , \dot{B}\right)$. By \Cref{lem:shallow-coordinate-test}, $\dot{a}_2 =\dot{b}_2 = 0$. Put \[
h\mathrel{:=}\begin{pmatrix}
 1 \\
 \lambda^r
\end{pmatrix} , \quad q\mathrel{:=}\left(1 , -\lambda^{- r}\right).
\] Since $w^r = w_1^r h$ and $u_1 = 0$, we have $b_1 =\mu q$ with $\mu \neq 0$. Cancelling $w_1^{r - 1}$ from $\dot{u}_1 = 0$ gives \begin{equation}
\left(\dot{b}_1 h\right)w_1 + r \mu \left(q \dot{w}\right) = \left(\dot{b}_1 h\right)w_1 + r \mu \dot{w}_1 = 0 .
\label{eq:one-dead-relation}
\end{equation} Thus $\dot{w}_1 = c w_1$ for a scalar $c$. If $c = 0$, \eqref{eq:shallow-dead-row} also gives $\dot{v}_1 = 0$, and the tangent is zero.

Suppose $c \neq 0$. \eqref{eq:shallow-dead-row} gives $v_1^{r - 1}\dot{v}_1 = - c w_1^r$, which implies the forms $v_1$, $\dot{v}_1$ and $w_1$ are all proportional. As $v_1 =\mu q x$, the row $\dot{b}_1$ is therefore proportional to $q$. Hence $\dot{b}_1 h = 0$, and \eqref{eq:one-dead-relation} forces $\dot{w}_1 = 0$, contradicting $c \neq 0$. \end{proof}

\begin{lemma}
Suppose all weight rows are nonzero, $w_2 =\lambda w_1$ with $\lambda \neq 0$, and $u_1 = u_2 = 0$. If $J_{\psi} \left(A , B\right)$ has a nonzero kernel, then $w_1$ is proportional to $x_1 -\lambda^{- r} x_2$. \label[lemma]{lem:two-dead-alignment} \end{lemma}

\begin{proof}
As in the previous proof, put \[
h\mathrel{:=}\begin{pmatrix}
 1 \\
 \lambda^r
\end{pmatrix} , \quad q\mathrel{:=}\left(1 , -\lambda^{- r}\right).
\] By assumption, $B h = 0$, so $b_1 =\mu q$ and $b_2 =\nu q$ with $\mu ,\nu \neq 0$. Put $\ell\mathrel{:=} q x$ and choose a nonzero kernel element $(\dot{A} ,\dot{B})$. At least one pair $(\dot{a}_i ,\dot{b}_i)$ is nonzero, so \Cref{lem:shallow-coordinate-test} forces $v_i$ and $w_i$ to be proportional for that index. Each $v_i$ is a nonzero multiple of $\ell$ and each $w_i$ is a nonzero multiple of $w_1$, so $w_1$ is proportional to $\ell$. \end{proof}

\begin{proof}[Proof of 
\Cref{prop:shallow-critical-classification}] A zero row is critical by \Cref{lem:zero-row-critical}, so assume all rows are nonzero. If $\left(A , B\right)$ is critical, \Cref{prop:regularity-unbiased} makes it degenerate. Thus some $u_i = 0$. A nonzero row of $B$ then gives a linear dependence between the nonzero powers $w_1^r , w_2^r$. Since powers of nonproportional linear forms are linearly independent, we obtain $w_2 =\lambda w_1$ with $\lambda \neq 0$. \Cref{lem:one-dead-neuron} forces $u_1 = u_2 = 0$, and \Cref{lem:two-dead-alignment} gives $w_1 =\rho(x_1 -\lambda^{- r} x_2)$. Because both rows of $B$ annihilate $\left(1 ,\lambda^r\right)^\top$, we deduce $B$ has the desired form. Since $\left(A , B\right)$ is real, the ratio $\lambda$ of the real linear forms $w_2$ and $w_1$ is real, and so are $\rho$ and the entries $\mu ,\nu$ of $B$.

Conversely, suppose $\left(A , B\right)$ are as in \eqref{eq:shallow-critical-form}. With $\ell = x_1 -\lambda^{- r} x_2$, we have $B(A x)^r = 0$. A direct computation gives \[
f\left(x\right)=\ell^r \begin{pmatrix}
 \rho^r +\mu^r \\
 \left(\rho \lambda\right)^r +\nu^r
\end{pmatrix} + x .
\] Keep $\lambda$ fixed and take the scalar tangent \[
\dot{\rho} =\rho , \quad \dot{\mu} = -\rho^r \mu^{1 - r} , \quad \dot{\nu} = -\left(\rho \lambda\right)^r \nu^{1 - r} .
\] Both coefficients in the displayed output have derivative zero. The parameter tangent is nonzero because $\dot{\rho} =\rho \neq 0$, so $\left(A , B\right)$ is critical. \end{proof}

\subsection{Singularities}
\label{sec:singularities}

We now show that the critical points studied above can cause singularities on the neurovariety. To define singularities in the sense of algebraic geometry formally, we recall that the \textit{Zariski tangent space} at $f \in \overline{\mathcal{M}}$ is \[
T_f \overline{\mathcal{M}}\mathrel{:=}\left\{v \in \mathbb{R}[x_1,\ldots,x_d]^d_{\leq r^L} : J_F (f)v = 0 \, \text{for every} \, F \in I\left(\mathcal{M}\right)\right\},
\] where $I\left(\mathcal{M} \right)$ is the ideal of polynomials vanishing on $\mathcal{M}$. The point $f$ is \textit{singular} on $\overline{\mathcal{M}}$ if \[
\dim T_f \overline{\mathcal{M}} >\dim \mathcal{M} ,
\] and \textit{nonsingular} otherwise.

Depending on their type, singularities may affect training dynamics and generalization \cite{marchetti2025position}. To study this hypothesis formally for a given architecture, a first step is to characterize  the singularities on the associated neurovariety. For monomial and polynomial convolutional networks, the singularities have been fully classified, whereas for plain fully connected networks only a partial description is available \cite{shahverdi2026learningrazorsedgeidentifiability,shahverdi2024geometry}. Here, we show that the neurovarieties of unbiased ResNets with an activation satisfying \eqref{eq:activation} are likewise singular. A complete classification, together with an analysis of the impact of these singularities, is left for future work.



\begin{proposition}
For the activation in \eqref{eq:activation}, if $L , d\geq 2$, then the unbiased neurovariety $\overline{\mathcal{M}}$ is singular at $\psi(I, \ldots, I, 0)$. \label[proposition]{prop:sing} \end{proposition}

We prove the proposition by constructing tangent directions. The construction uses the first nonzero coefficient of a curve, which belongs to the Zariski tangent space even when the curves ordinary derivative vanishes.

\begin{lemma}
Let $p \in \overline{\mathcal{M}}$. If a polynomial curve in $\overline{\mathcal{M}}$ has expansion $\gamma(t)= p + t^n v + O\left(t^{n + 1}\right)$ with $n\geq 1$, then $v \in T_p \overline{\mathcal{M}}$. \label[lemma]{lem:arc} \end{lemma}

\begin{proof}
For $F \in I\left(\overline{\mathcal{M}}\right)$, expansion gives \[
0 = F\left(\gamma(t)\right)= t^n J_F \left(p\right)v + O\left(t^{n + 1}\right).
\] Thus $J_F \left(p\right)v = 0$ for every defining equation $F$, as required. \end{proof}

Let $s\mathrel{:=}\min\left\{k : a_k \neq 0\right\}$, so $2\leq s\leq r$. If $f$ is any depth-$\left(L - 1\right)$ ResNet, applying another residual block with the zero weight matrix preserves $f$; hence $f$ is realized as a point of the depth-$L$ neurovariety $\overline{\mathcal{M}}$. The next lemma describes directions obtained by varying that final matrix.

\begin{lemma}
Let $f =\left(F_1 ,\ldots , F_d\right)$ be the output of an unbiased ResNet of depth $L - 1$ and width $d$, viewed as a point of the depth-$L$ unbiased neurovariety $\overline{\mathcal{M}}$. Define \[
W_f\mathrel{:=}\left\{H\left(f\right): H :\mathbb{R}^d \rightarrow \mathbb{R}^d \text{is homogeneous of degree} s\right\}.
\] Then $W_f \subset T_f \overline{\mathcal{M}}$ and $\dim W_f = d \binom{s + d - 1}{d - 1}$. \label[lemma]{lem:final-layer-tangents} \end{lemma}

\begin{proof}
Fix the $L - 1$ weight matrices producing $f$, and fix a matrix $A$. Append an $L$th residual block with weight matrix $t A$. This defines a polynomial curve $\gamma :\mathbb{R} \rightarrow \overline{\mathcal{M}}$ by \[
\gamma(t)\left(x\right)\mathrel{:=} f\left(x\right)+\sigma(t A f\left(x\right)).
\] Since it has expansion \[
\gamma(t)= f + a_s t^s \left(A f\right)^s + O\left(t^{s + 1}\right),
\] \Cref{lem:arc} yields $\left(A f\right)^s \in T_f \overline{\mathcal{M}}$ for every matrix $A$.

Powers of linear forms span all homogeneous forms of degree $s$, as one can easily verify. Because the rows of $A$ vary independently, the vectors $\left(A f\right)^s$ consequently span $W_f$.

Substitution $H \mapsto H\left(f\right)$ is injective by the algebraic independence of $F_1 ,\ldots , F_d$ in \Cref{lemma:lin-indep-fl}. There are $\binom{s + d - 1}{d - 1}$ monomials of degree $s$ in $d$ variables, giving the stated dimension for $d$ output coordinates. \end{proof}

\begin{proof}
\Cref{prop:sing} Let $f =\left(F_1 ,\ldots , F_d\right)$ be the depth-$\left(L - 1\right)$ output at the identity weights. Appending a residual block with the zero matrix gives the point $f =\psi(I, \ldots, I, 0)$ of $\overline{\mathcal{M}}$.

Varying only the first $L - 1$ matrices gives a subspace $V \subset T_f \overline{\mathcal{M}}$. The identity parameter is nondegenerate, so \Cref{prop:regularity-unbiased} gives $\dim V =(L - 1)d^2$. Every vector in $V$ has degree at most $r^{L - 1}$.

The subspace $W_f$ from \Cref{lem:final-layer-tangents} also lies in $T_f \overline{\mathcal{M}}$. At the identity weights, the leading component of the coordinate $F_i$ is $c x_i^{r^{L - 1}}$ for some nonzero scalar $c$ independent of $i$ (\Cref{lem:leading-components}). Distinct degree-$s$ monomials in the $F_i$ have distinct leading monomials in $x$, so every nonzero vector in $W_f$ has degree $s r^{L - 1} > r^{L - 1}$. Hence $V \cap W_f =\left\{0\right\}$.

Since $d\geq 2$, we have \[
\binom{s + d - 1}{d - 1} > d .
\] Together with \Cref{cor:dimension}, this gives \[
\dim T_f \overline{\mathcal{M}}\geq \left(L - 1\right)d^2 + d \binom{s + d - 1}{d - 1} > L d^2 =\dim \overline{\mathcal{M}} .
\] Thus $f$ defines a singular point on $\overline{\mathcal{M}}$. \end{proof}

The identity weights ensure that the leading monomials in the last argument cannot cancel. Nondegeneracy of the preceding layers alone would not give the required disjointness of tangent subspaces.

\begin{example}
Let $L = d = r = 2$ and $\sigma(z)\mathrel{:=} z^2$. Take \[
\ell\mathrel{:=} x_1 - x_2 , \quad W_1\mathrel{:=}\begin{pmatrix}
 1 & - 1 \\
 1 & - 1
\end{pmatrix} , \quad f\mathrel{:=}\begin{pmatrix}
 \ell^2 + x_1 \\
 \ell^2 + x_2
\end{pmatrix} .
\] The one-layer parameter $W_1$ is nondegenerate. With $F_1 , F_2$ denoting the coordinates of $f$, the space $W_f$ contains \[
\begin{pmatrix}
 \left(F_1 - F_2\right)^2 \\
 0
\end{pmatrix} =\begin{pmatrix}
 \ell^2 \\
 0
\end{pmatrix} .
\] The ordinary derivative space $V$ contains the same vector: a direct computation shows that perturbing $W_1$ in the direction \[
\dot{W}_1\mathrel{:=}\frac{1}{2} \begin{pmatrix}
 1 & - 1 \\
 0 & 0
\end{pmatrix}
\] changes the output by $t\left(\ell^2 , 0\right)^\top + O\left(t^2\right)$. Thus $V \cap W_f \neq \left\{0\right\}$. This example explains the choice of the identity weights in \Cref{prop:sing}. \label[example]{ex:singular-choice} \end{example}

We conclude with a visualization of the scalar case.

\begin{example}
For $L = 2$, $d = 1$, and $\sigma(z) := z^2$, the coefficient map is
\[
\psi(a,b) = (c_2,c_3,c_4) = (a^2+b^2,\,2a^2b^2,\,a^4b^2).
\]
Eliminating $a,b$, as can be done with suitable software, gives the defining ideal of the neurovariety, generated by
\begin{equation}
F(c_2,c_3,c_4) := c_3^3 - 4c_2c_3c_4 + 8c_4^2.
\label{eq:scalar-neurovariety}
\end{equation}
Then one can check $\overline{\mathcal{M}} = V(F)$, and that
\[
\operatorname{Sing}(\overline{\mathcal{M}})
= \{(c_2,0,0):c_2\in\mathbb{R}\}.
\]
Its inverse image is $\{ab=0\}$, which is also the critical locus of $\psi$ by direct differentiation. Its intersection with $\mathcal{M}$ is the half-line $c_2\geq 0$.

\Cref{fig:scalar-neurovariety} shows the real image and a section of its Zariski closure. Setting $s := a^2$ and $t := b^2$ gives
\[
\mathcal{M}=\{(s+t,2st,s^2t):s,t\geq 0\}.
\]
In the plane $c_2=1$, the image is the loop $(c_3,c_4)=(2s(1-s),s^2(1-s))$, $0\leq s\leq 1$. Allowing all real $s$ gives the full real section of $V(F)$; the two values $s=0,1$ meet at its singular point $(0,0)$.
\label[example]{ex:scalar-geometry}
\label[example]{ex:scalar-singular}

As \Cref{fig:scalar-neurovariety} shows, the singular locus of $\mathcal{M}$ is exposed in the sense of \cite{ribot2026exposed,kohn2024function, shahverdi2026learningrazorsedgeidentifiability}; that is, the union of the Voronoi cells of the singular points is full-dimensional. Consequently, if a point $u$ in the ambient space of $\mathcal{M}$ is sampled from a distribution with full support, then with positive probability the distance from $u$ to $\mathcal{M}$ is attained at a singular point. Since minimizing the squared-error loss amounts to minimizing a distance \cite{marchetti2025critical}, this suggests that singular points play an important role in the training dynamics and generalization of ResNets.
\end{example}

\begin{figure}[htbp]
\centering
\includegraphics[width=\linewidth]{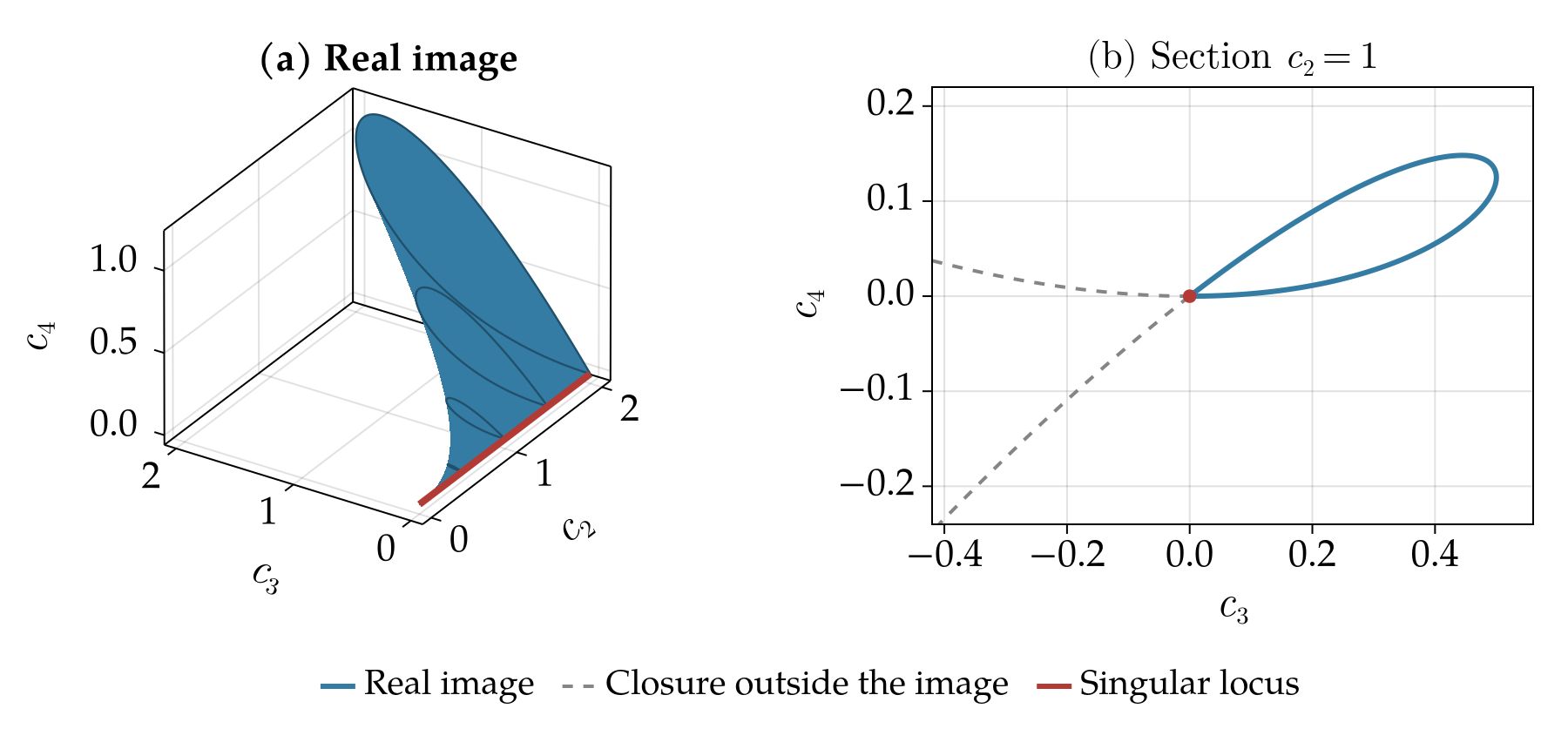}
\caption{Scalar quadratic ResNets with two layers.
Left: the real image $\mathcal{M}$ for $0\leq c_2\leq 2$, with its singular half-line in red.
Right: the section $c_2=1$; blue is the real image, and dashed gray branches belong to its Zariski closure but not to the image. The red point lies on the singular line of $V(F)$.}
\label{fig:scalar-neurovariety}
\end{figure}

\section{Conclusion}

We have shown that skip connections break the parameter symmetries of polynomial networks. Away from a degenerate locus, the fibers of the ResNet parametrization are singletons for generic activations of degree at least five, or consist of root-of-unity rescalings for a broad class of activations, and the differential is injective. As an application, we exhibited singular points on the unbiased neurovariety.


Several questions remain open. The fibers over degenerate parameters are not yet understood; they may consist of further discrete orbits or contain positive-dimensional families (\Cref{sec:degfib}). It would also be of interest to characterize the critical points of the parametrization and the singularities of the neurovariety completely, and, from the perspective of machine learning, to analyze their effect on training and generalization. Finally, a natural extension of our results is to networks of unequal widths, via projected skip connections.

\section*{Provenance statement}

The authors produced all proofs themselves, with the exception of the proof of the existence of singularities (\Cref{prop:sing}), which is due to Claude Fable 5. Claude Sonnet 5.5 was used to polish the writing, and Claude Opus 5.5 was used to produce the plots in \Cref{fig:scalar-neurovariety}.

\section*{Acknowledgements}

We are grateful to Vahid Shahverdi for  discussions on this topic. The authors were supported by the Wallenberg AI, Autonomous Systems and Software Program (WASP) funded by the Knut and Alice Wallenberg Foundation. GLM was supported by the AI4S initiative funded by the Knut and Alice Wallenberg Foundation.

\bibliographystyle{plain}
\bibliography{references}

\end{document}